\documentclass[10pt, conference, letterpaper]{IEEEtran}

\usepackage[T1]{fontenc}
\ifCLASSINFOpdf
\else
\fi
\usepackage[cmintegrals]{newtxmath}
\usepackage{url}
\usepackage{cite}
\usepackage{amsmath}

\usepackage{amsthm}
\usepackage{mathtools}
\usepackage{algorithmic}
\usepackage[ruled,linesnumbered]{algorithm2e}
\usepackage{booktabs}
\usepackage{multirow}
\usepackage{graphicx}
\usepackage{textcomp}
\usepackage{xcolor}
\usepackage{bm}
\usepackage{enumitem}
\usepackage{subfigure}
\usepackage{url}
\usepackage{hyperref}
\usepackage[normalem]{ulem}  

\newtheorem{assumption}{Assumption}

\newtheorem{theorem}{Theorem}

\newcommand{\eg}{\textit{e}.\textit{g}.}

\newcommand{\bfa}{\mathbf{a}}
\newcommand{\bfs}{\mathbf{s}}

\newcommand{\cmmnt}[1]{}
\newcommand{\tr}[1]{}

\newcommand{\transitions}{T}
\newcommand{\bellman}{\mathcal{B}}

\allowdisplaybreaks[4]

\newboolean{showcomments}
\setboolean{showcomments}{true}
\newcommand{\comment}[1]{\ifthenelse{\boolean{showcomments}}
	{\textcolor{red}{(Comment: #1)}}
	{}
}
\newcommand{\red}[1]{\ifthenelse{\boolean{showcomments}}
	{\textcolor{red}{#1}}
	{}
}

\newboolean{showanswers}
\setboolean{showanswers}{true}
\newcommand{\answer}[1]{\ifthenelse{\boolean{showanswers}}
	{\textcolor{blue}{(Answer: #1)}}
	{}
}

\SetKwRepeat{Do}{do}{while}

\IEEEoverridecommandlockouts

\begin{document}

\title{FORLER: Federated Offline Reinforcement Learning with Q-Ensemble and Actor Rectification\vspace{-0pt}}

\author{
\IEEEauthorblockN{
Nan Qiao$^{*}$,
Sheng Yue$^{\dagger\ddagger}$
}
\IEEEauthorblockA{
$^{*}$Central South University, School of Computer Science, Changsha, China\\
$^{\dagger}$Sun Yat-sen University, School of Cyber Science and Technology, Shenzhen, China\\
Email: nan.qiao@csu.edu.cn, yuesh5@mail.sysu.edu.cn\\
$^{\ddagger}$Corresponding Author
}
}


\maketitle

\vspace{-5cm}
\begin{abstract}
In Internet-of-Things systems, federated learning has advanced online reinforcement learning (RL) by enabling parallel policy training without sharing raw data. However, interacting with real environments online can be risky and costly, motivating offline federated RL (FRL), where local devices learn from fixed datasets. Despite its promise, offline FRL may break down under low-quality, heterogeneous data. Offline RL tends to get stuck in local optima, and in FRL one device’s suboptimal policy can degrade the aggregated model, i.e., policy pollution.
We present FORLER, combining Q-ensemble aggregation on the server with actor rectification on devices. The server robustly merges device Q-functions to curb policy pollution and shift heavy computation off resource-constrained hardware without compromising privacy. Locally, actor rectification enriches policy gradients via a zeroth-order search for high-Q actions plus a bespoke regularizer that nudges the policy toward them.
A $\delta$-periodic strategy further reduces local computation. We theoretically provide safe policy improvement performance guarantees. Extensive experiments show FORLER consistently outperforms strong baselines under varying data quality and heterogeneity.

\end{abstract}

\section{Introduction}
\label{sec:introduction}
Sequential decision-making is a fundamental capability underpinning Internet of Things applications, from autonomous driving to multi-robot control \cite{zhang2025carplanner,qiao2023popec,Wang2025Squeezer,Zhang2023Distributed}.
Recent progress has been driven by deep reinforcement learning \cite{wang2023uav}, with a major shift from single- to multi-agent DRL that improves exploration \cite{seid2021multi}. A persistent challenge, however, is the need to transmit substantial data to a central controller, stressing bandwidth and risking privacy \cite{qi2021federated}. This has motivated federated reinforcement learning (FRL) \cite{jin2022federated}.
In FRL, local devices interact with the environment online, update policies locally, and periodically aggregate them on a server \cite{fang2025provably}. However, such online interaction is often impractical, costly, or unsafe in real systems. For example, training autonomous-driving policies on public roads poses substantial safety risks \cite{levine2020offline}. Offline FRL mitigates these issues by learning from fixed local datasets with offline RL methods \cite{qiao2025fova,rengarajan2023federated,fujimoto2021minimalist,qiao2025fova,an2021uncertainty,qiao2026less},
Yet, it remains challenging at both the local update and global aggregation stages.



At the local update stage, the value function guides policy improvement, but the objective is non-concave, so policy-gradient methods often get stuck in local optima, especially with low-quality data \cite{pan2022plan}. To stabilize offline training, many works add behavior cloning \cite{kim2025penalizing,fujimoto2021minimalist}, which implicitly assumes behavior actions are near-optimal \cite{rengarajan2023federated}.
When data quality is uneven, coverage is limited, or the BC weight is large, behavior cloning ties the policy to the behavior policy, dampens Q-driven improvement, and, under federated aggregation, increases the risk of suboptimality and pollution.

Furthermore, current global aggregation methods have notable limitations, as our analysis shows.
As depicted in Fig.~\ref{fig:motivating}, the performance of the advanced method degrades markedly when devices with low-quality datasets join training. 
These devices learn policies that are locally optimal to their data but {pollute} the global policy during aggregation. We term this \textit{policy pollution}.
Moreover, many offline FRL methods deliberately push more computation onto devices to preserve privacy and reduce uplink traffic \cite{rengarajan2023federated,qiao2025fova}. This design burdens resource-constrained IoT hardware, increasing latency and energy costs and worsening instability when compute capabilities are heterogeneous.
The polluted global policy then harms other devices, degrading otherwise strong local policies.



To address these issues, we propose federated offline reinforcement learning with Q-ensemble and actor rectification (FORLER), featuring two components: {Q-ensemble aggregation} and {actor rectification}. During server-side aggregation, we collect Q-parameters from all devices and run a new offline RL algorithm based on Q-ensembles, which both improves offline RL itself and robustly aggregates cross-device information to produce updated Q-functions and policies. On devices, actor rectification exploits conservative value functions by augmenting standard policy gradients with a zeroth-order search for high-Q actions; a tailored regularizer then pulls the policy toward these actions. We also adopt a ``$\delta$-periodic'' strategy that samples only every $\tau$ updates to reduce local computation.

In a nutshell, our main contributions are summarized below.
\begin{itemize}
    \item We design an ensemble-based offline RL algorithm for server-side aggregation that mitigates {policy pollution} and better accommodates statistical heterogeneity.
    \item We introduce actor rectification for local training, alleviating local optima.
    The ``$\delta$-periodic'' strategy further improves computational efficiency.
    \item We prove that FORLER enjoys the safe policy improvement guarantee and demonstrate consistent gains over prior methods through extensive experiments.
\end{itemize}

\section{Preliminaries}\label{sec:preliminaries}
This section reviews federated learning and offline reinforcement learning, and then presents the system model.

\paragraph{FL and RL}
Federated learning trains models locally on devices and aggregates their updates on a central server without sharing raw data. In FedAvg, the global objective is
\(
F(w)=\sum_{k=1}^{\mathcal{K}} p_k\, F_k(w),
\)
with dataset-size weights
$
p_k = \frac{|\mathcal{D}_k|}{|\mathcal{D}|},
|\mathcal{D}| = \sum_{k=1}^{\mathcal{K}} |\mathcal{D}_k|.
$
However, size-based weighting may misrepresent data quality, and simple linear aggregation can perform poorly under heterogeneous data and system conditions.
Reinforcement learning can be formalized as an MDP $\mathcal{M}=(\mathcal{S},\mathcal{A},P,R,\mu_0,\gamma)$ with goal $\max_{\pi}\ \mathbb{E}\left[\sum_{t=0}^{\infty}\gamma^t r(s_t,a_t)\right]$. Q-learning estimates $Q^*(\bfs,\bfa)$; actor–critic methods update a policy (actor network) using a value estimator (critic network) by minimizing the Bellman error $\left(Q(\bfs,\bfa)-\mathcal{B}^{\pi}Q(\bfs,\bfa)\right)^2$ with
$\mathcal{B}^{\pi}Q(\mathbf{s},\mathbf{a})=\mathbb{E}_{\bfs'\sim T(\cdot|\bfs,\bfa)}\left[r(\mathbf{s},\mathbf{a})+\gamma\ \mathbb{E}_{\mathbf{a'}\sim\pi(\cdot|\mathbf{s'})} Q_{\phi}(\mathbf{s'},\mathbf{a'})\right]$.
In offline RL, learning minimizes this error over a fixed dataset, while the policy maximizes the critic’s predicted return.
TD3BC encourages staying on the data manifold \cite{fujimoto2021minimalist}:
\begin{equation}
\label{eq:td3+bc}
    \pi_{td3bc} \leftarrow \arg\max_{\pi}\ \mathbb{E}_{(\bfs,\bfa)\sim\mathcal{D}}\left[\lambda Q(\bfs,\pi(\bfs))-(\pi(\bfs)-\bfa)^2\right].
\end{equation}
where $\lambda$ trades off value maximization and behavior cloning.
Alternatively, CQL learns conservative value estimates \cite{cql}:
\begin{align}
Q^{t+1}_\mathit{cql} &\leftarrow \arg \min _{Q}\ \mathbb{E}_{\bfs,\bfa,s^{\prime}\sim\mathcal{D}}\left[\big(Q(\bfs,\bfa)-\hat{\mathcal{B}}^{\pi} Q^{t}(\bfs,\bfa)\big)^{2}\right] \nonumber\\
&+\ \omega\left(\mathbb{E}_{s\sim\mathcal{D},\ a\sim\mu(\cdot|\bfs)}[Q(\bfs,\bfa)]-\mathbb{E}_{\bfs,\bfa\sim\mathcal{D}}[Q(\bfs,\bfa)]\right),
\label{eq:cql}
\end{align}
where $\hat{\mathcal{B}}^{\pi} Q = r(\mathbf{s},\mathbf{a})+\gamma\ \mathbb{E}_{\mathbf{a'}\sim\pi(\cdot|\mathbf{s'})}\left[\min_{n=1,2} Q_{n}(\mathbf{s'},\mathbf{a'})\right]$ and $\omega$ is a positive coefficient. Such regularization curbs overestimation yet can be sensitive to the number of critics.

\paragraph{Deployment and dataset assumptions.}
We illustrate the system model in Fig.~\ref{fig:system-model}.
We consider an IoT/edge system with $\mathcal{K}$ devices coordinated by a central server.
Each device $k$ maintains a fixed offline dataset $\mathcal{D}_k$ collected from historical interactions under a behavior policy $\pi_{\beta_k}$, while raw transitions never leave the device.
Concretely, device $k$ stores
$\mathcal{D}_k=\{(\bfs_j,\bfa_j,r_j,\bfs'_j)\}_{j=1}^{|\mathcal{D}_k|}$,
and the server stores an auxiliary offline dataset
$\mathcal{D}_o=\{(\bfs_j,\bfa_j,r_j,\bfs'_j)\}_{j=1}^{|\mathcal{D}_o|}$,
which can be a public/proxy dataset or non-sensitive logs owned by the operator and is used only for server-side offline updates.
Training proceeds in synchronous communication rounds:
the server broadcasts the current global parameters $(Q_o^t,\pi_o^t)$,
each device performs $\Gamma$ local offline updates with actor-side rectification based on $\mathcal{D}_k$,
uploads only model parameters (e.g., critic parameters) to the server,
and the server performs offline RL with a Q-ensemble on $\mathcal{D}_o$ to update global parameters and then broadcasts the updated models back to devices.
\begin{figure}[htbp]
    \vspace{-0cm} 
    \setlength{\abovecaptionskip}{-0cm} 
    \setlength{\belowcaptionskip}{-0cm} 
    \centering
	\includegraphics[width=0.345\textwidth]{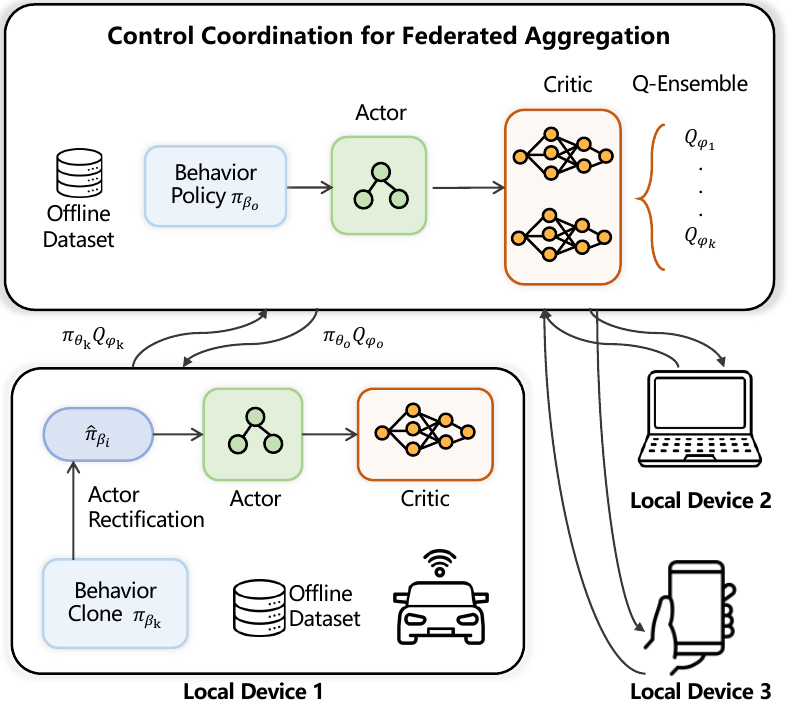}
    \vspace{-0pt}
	\caption{The architecture of FORLER.  
    The system integrates a suite of devices coordinated by a server. Both the individual devices and the server are equipped with offline datasets
    and compute independently.
    }
	\label{fig:system-model}
    \vspace{-10pt}
\end{figure}

\begin{figure}[htbp]
    \centering
    \hspace{-15pt}
    \subfigure{
                \begin{minipage}[htbp]{0.555\columnwidth}
                    \centering	
                    \label{fig:c1}
                    \includegraphics[width=\textwidth]{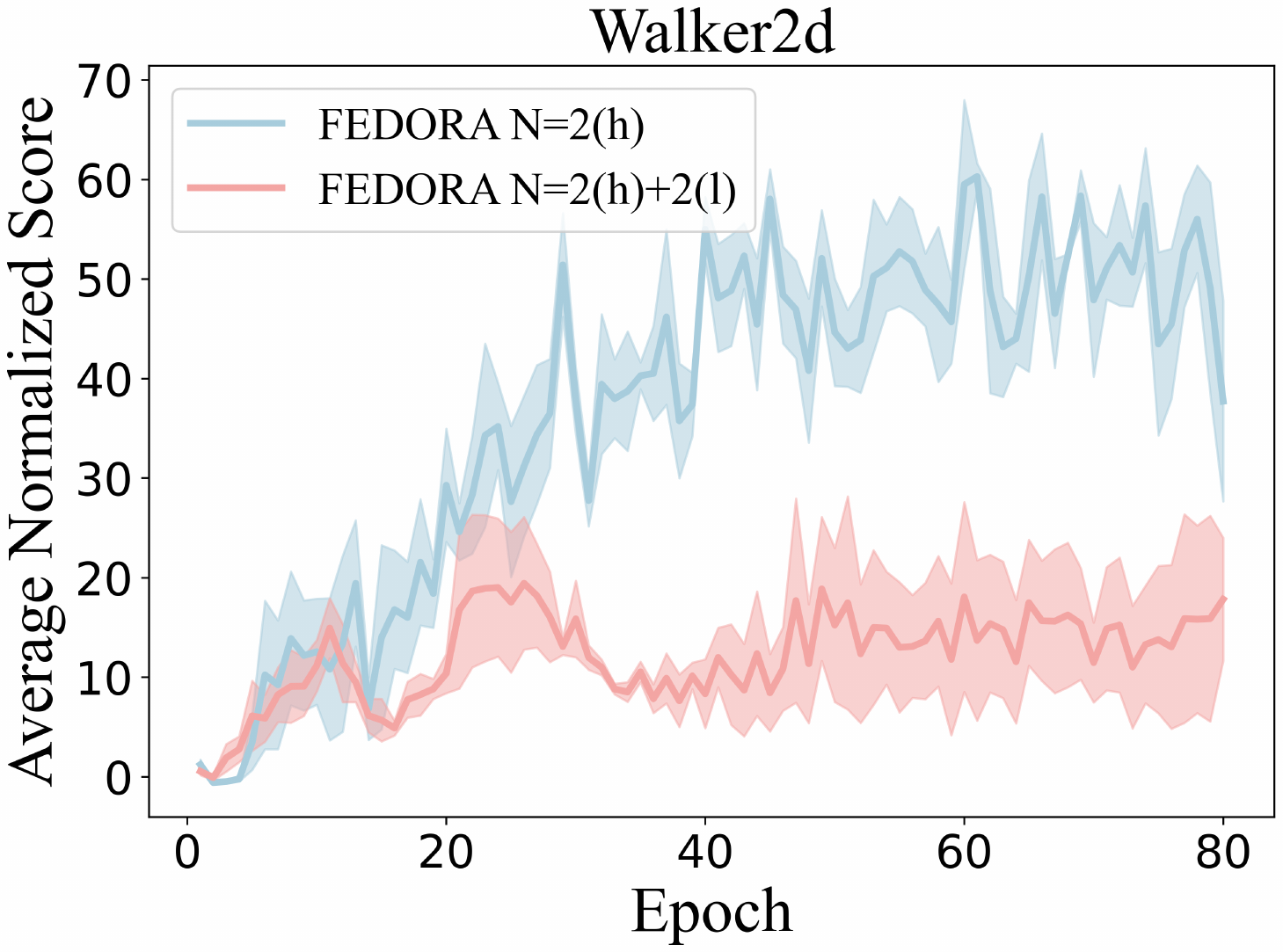}
              	\end{minipage}
        }
        \hspace{-0.02\textwidth}
        \subfigure{
                \begin{minipage}[htbp]{0.400\columnwidth}
                    \centering	
                    \label{fig:c3}
                    \includegraphics[width=\textwidth]{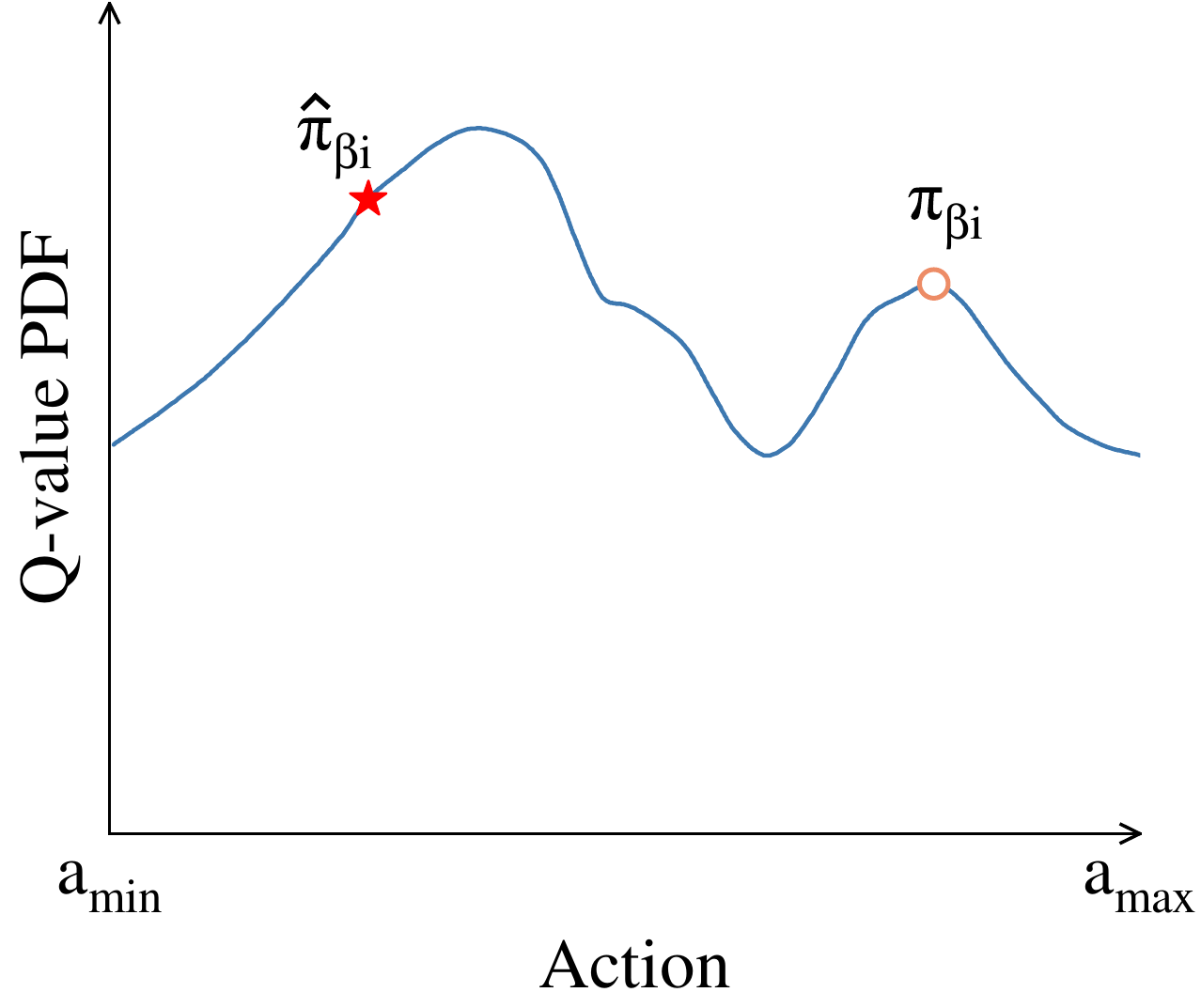}
              	\end{minipage}
        }
    \vspace{-5pt}

	\caption{The performance of advanced algorithms, \eg~FEDORA.
 (Left) Overall performance difference between high-quality datasets and mixed (high-quality + additional low-quality) datasets.
 (Right) A visual representation of the Q-value vs. Action.  The circle symbolizes the predicted action derived from the updating policy of FEDORA, while the pentacle represents the corresponding action predicted by the updating policy of our method.
 }
\label{fig:motivating}
\vspace{-15pt}
\end{figure}

\section{Proposed Method}\label{sec:proposed_method}
In this section, we first introduce a motivating example to show that advanced algorithms suffer from policy pollution and become trapped in local optima. We then introduce FORLER, which couples Q ensemble aggregation on the server with $\delta$ periodic actor rectification for local updates. This dual design addresses these shortcomings and yields a more robust framework for offline FRL.

\subsection{The Motivating Example}\label{subsec:motivate}

We conducted a series of experiments using the walker2d dataset to assess the impact of data quality on learning performance. Specifically, we utilized two devices, each operating on medium-expert data with a dataset size $|\mathcal{D}_k|=20k$. 
Our findings, illustrated in the left part of Fig.\ref{fig:motivating}, reveal a counter-intuitive result: despite the extra dataset added in the second experiment, its performance was significantly lower compared to the first. 
These results highlight policy pollution and its impact, with the extra low-quality data degrading the learned policy and overall performance.

Additionally, we identify a key limitation of advanced methods: critic-only regularization does not enable devices to learn effective coordination policies. The right panel of Fig.\ref{fig:motivating} shows the Q value landscape at a given training step. The circle denotes the action produced by a representative advanced method under its current update, while the pentacle denotes our approach. The former becomes trapped in a poor local optimum. As a result, devices fail to leverage the conservative value function globally, leading to suboptimal policies.

\subsection{Local Updates with $\delta$-periodic Actor Rectification}

\paragraph{Local policy evaluation}
Obviously, extrapolation error is a tricky problem in local offline policy evaluation.
Because training does not interact with the environment online, the state–action distribution induced by the current policy diverges from the dataset distribution\cite{cql,fujimoto2021minimalist,an2021uncertainty}. 
We therefore adopt CQL to mitigate this shift and curb overestimation on out-of-distribution actions, which is as follows:
\begin{align}
\label{eq:localQ}
Q_{k}^{\tau+1} \leftarrow \arg \min _{Q} {\frac{1}{2} \mathbb{E}_{\bfs, \bfa, s^{\prime} \sim \mathcal{D}_{k}}\left[\left(Q(\bfs, \bfa)-\hat{\mathcal{B}}^{\pi_k} Q_{k}^{k}(\bfs, \bfa)\right)^{2}\right]}\nonumber\\
+{\omega_c\left(\mathbb{E}_{s \sim \mathcal{D}_{k}, a \sim \mu(\cdot \mid s)}[Q(\bfs, \bfa)]-\mathbb{E}_{\bfs, \bfa \sim \mathcal{D}_{k}}[Q(\bfs, \bfa)]\right)},
\end{align}
which is similar with Eq.~\eqref{eq:localQ} and has been fully covered.

\paragraph{Local policy improvement with $\delta$-periodic Actor Rectification}

In light of the challenges in the local optima problem, we adopt the following policy-improvement process:
\begin{align}
\label{eq:localpi}
    \pi_k^{\tau+1} &\leftarrow \arg \max _{\pi} \mathbb{E}_{s \sim \mathcal{D}_{k}, a \sim \pi(\cdot \mid s)}\left[\hat{Q}_{k}(\bfs, \bfa)\right]\nonumber\\
    &-\alpha_1\left(\pi-{\color{blue}\hat{\pi}_{\beta_k}^{\tau}}\right)^{2}-\alpha_2\left(\pi-\pi_0\right)^{2},
\end{align}
where $\hat{\pi}_{\beta_k}^{\tau}$ represents the action proposed by a zeroth-order optimizer, with $\alpha_1$ and $\alpha_2$ acting as positive coefficients. 
In Eq.~\ref{eq:localpi}, the term weighted by $\alpha_1$ pulls the actor toward the rectified action $\hat{\pi}_{\beta_k}$, which is obtained by a derivative-free search for high-$Q$ actions and helps escape local optima under low-quality data. The term weighted by $\alpha_2$ anchors the policy to $\pi_0$ (the received global policy at the beginning of the round) to prevent excessive policy drift when $Q$ estimates are conservative/noisy.
Notably, $\left(\pi-{a}\right)^{2}$ have been utilized in TD3BC and FEDORA \cite{fujimoto2021minimalist,rengarajan2023federated}, where $a$ is the `seen' action in the dataset.
Nevertheless, we replace it with $\hat{\pi}_{\beta_k}^{\tau}$, which is the `optimized' action and makes our method perform well even if the quality of the dataset is low.
Specifically, in the training process, we initialize a Gaussian distribution $\mathcal{N}({x_k}, {y_k})$. Over $I$ iterations, generate a population of $N$ individual actions $\mathcal{\hat{A}}_k=\{\hat{\bfa}_{k}^n\}^{N}_{n=1}$ sampled from the distribution $\mathcal{N}({x_k}, {y_k})$. For each action in the population, it estimates $Q$-values using the critic $Q^{\tau}_k$.
Subsequently, the parameters ${x_k}$ and ${y_k}$ of the Gaussian distribution are updated according to a predefined update equation:
\begin{align}
x_{k}^{i+1}=\sum_{n=1}^{N} w_n\, \hat{\bfa}_{k}^{n},
\quad
y_{k}^{i+1}=\sqrt{\sum_{n=1}^{N} w_n \left\|\hat{\bfa}_{k}^{n}-x_{k}^{i+1}\right\|_2^{2}} .
\label{eq:dist_update}
\end{align}
where $w_n = \frac{\exp(\beta\, Q_k^\tau(s,\hat{\bfa}_{k}^{n}))}{\sum_{m=1}^{N}\exp(\beta\, Q_k^\tau(s,\hat{\bfa}_{k}^{m}))}$.
We select the optimal candidate action $\hat{\pi}^{\tau+1}_{\beta_k}$ based on the highest $Q$-value among the sampled actions and the current policy. 
Nevertheless, the described method poses significant challenges due to its \textit{high computational complexity and extensive sampling} requirements.   These aspects render it less suitable for environments with constrained computing and storage capabilities, particularly for devices. 
To address this limitation, we introduce an alternative approach, which we refer to as \textit{$\delta$-periodic Actor Rectification}.   This new strategy involves computing $\hat{\pi}_{\beta_k}^{\tau}$ at regular intervals, each time $\delta$ elapses.  
Accordingly, $\hat{\pi}_{\beta_k}^{\tau}$ is updated as
\begin{align}
    \hat{\pi}^{\tau+1}_{\beta_k} \doteq \left\{\begin{array}{ll}
    \arg\max_{\hat{\pi}_{\beta_k} \in \mathcal{\hat{A}}_k \cup \pi_k(\bfs)} Q_k^{\tau}(\bfs,\hat{\pi}_{\beta_k}(\bfs)), & \tau \in \Gamma^{\delta}, \\
    \arg\max_{\hat{\pi}_{\beta_k} \in \hat{\pi}_{\beta_k}^{\tau} \cup \pi_k(\bfs)} Q_k^{\tau}(\bfs,\hat{\pi}_{\beta_k}(\bfs)), & \text {otherwise,}
    \end{array}\right.\nonumber
\end{align}
where $\Gamma^{\delta}\doteq\{0,\delta,2\delta,\cdots\}$.
The method avoids computing $\hat{\pi}_{\beta_k}^{\tau}$ at every iteration, thereby simplifying the procedure and substantially reducing computation and sampling overhead.  
Moreover, beyond these practical benefits, we have also established through theoretical analysis that this method enhances the security policy, ensuring a more robust and efficient performance.
The complete pseudo code of \textit{$\delta$-periodic Actor Rectification} is shown in Algorithm \ref{alg:ar}.
The zeroth-order search in actor rectification evaluates the critic on $N$ sampled actions for $I$ iterations per update, leading to $\mathcal{O}(I\cdot N)$ Q-forward evaluations.
With full rectification ($\delta = 1$), this cost is incurred at every local step, whereas the proposed $\delta$-periodic strategy performs the full search only every $\delta$ steps and reuses the rectified action otherwise, reducing the rectification-related evaluations by approximately a factor of $\delta$.
Communication payload per round remains unchanged because the same model parameters are exchanged. The benefit is primarily on-device computation, latency and energy.

\begin{algorithm}[htbp]
    \caption{$\delta$-periodic Actor Rectification}
    \label{alg:ar}
    \LinesNumbered
    \KwIn{Critic $Q^{\tau}_k$, $\bfs$, $\hat{\pi}_{\beta_k}^{\tau}$, and $\tau$ from Algorithm \ref{alg:forler}}

        \eIf{$\tau \in \Gamma^{\delta}$}{
            Initialize $\mathcal{N}({x_k}, {y_k})$\\
            \For{iteration $i=1$ to $I$}{
                Construct $\hat{\mathcal{A}}_k$ by sampling $\hat{\bfa}_{k}^n\sim\mathcal{N}(x_k, y_k)$ for $n=1,\ldots,N$.\\
                Estimate $Q$-values for $K$ individuals in the population $\{Q_{k}^1(\bfs, \hat{\bfa}_{k}^n)\}_{n=1}^N$\\
                Update ${x_k}$ and ${y_k}$ according to Eq. (\ref{eq:dist_update})\\
            }
            Obtain the picked candidate action $\hat{\pi}^{\tau+1}_{\beta_k}=\arg\max_{\hat{\pi}_{\beta_k} \in \mathcal{\hat{A}}_k \cup \pi_k(\bfs)} Q_k^{\tau}(\bfs,\hat{\pi}_{\beta_k}(\bfs))$\\
            }
            {
            { $\hat{\pi}^{\tau+1}_{\beta_k}=\arg\max_{\hat{\pi}_{\beta_k} \in \hat{\pi}_{\beta_k}^{\tau} \cup \pi_k(\bfs)} Q_k^{\tau}(\bfs,\hat{\pi}_{\beta_k}(\bfs))$}\\
            }
            {}
    \KwOut{$\hat{\pi}^{\tau+1}_{\beta_k}$}
\end{algorithm}

\subsection{Server Aggregation with Q-Ensemble}
We design a server side Q ensemble for aggregation on the server dataset $\mathcal{D}_o$. Each device contributes two critic heads, hence $2K$ in total. Inspired by SAC-N in offline RL \cite{an2021uncertainty}, the server builds a pessimistic target with the pointwise minimum and updates the iterates $\pi_o^t$ and $Q_o^t$ to curb policy pollution and overestimation.
Accordingly, the critic is updated by
\begin{align}
\label{eq:globalQ}
Q_o^{t+1} &~\leftarrow~ \arg \min_{Q}\;
\mathbb{E}_{\bfs, \bfa, \bfs' \sim \mathcal{D}_{o}}
\Big[ \big( Q(\bfs, \bfa) - \widehat{\mathcal{B}}^{\pi_o}_{K} Q(\bfs, \bfa) \big)^2 \Big] \\
&\quad +~ \omega^s \Big( \mathbb{E}_{\bfs \sim \mathcal{D}_o,~ \bfa' \sim \pi_{o}(\cdot \mid \bfs)} Q(\bfs, \bfa')
~-~ \mathbb{E}_{\bfs,\bfa \sim \mathcal{D}_o} Q(\bfs, \bfa) \Big), \nonumber
\end{align}
with empirical Bellman operator
$
\widehat{\mathcal{B}}^{\pi_o}_{K} Q(\bfs, \bfa)
= r(\bfs, \bfa) + \gamma\cdot\mathbb{E}_{\bfa' \sim \pi_{o}(\cdot \mid \bfs')}
\left[ \min_{i=1,\ldots,2K} Q_{i}(\bfs', \bfa') - \beta \log \pi_{o}(\bfa' \mid \bfs') \right],
$
where $\omega^s\ge0$ weights a conservative behavior regularizer and $\beta\ge0$ is the entropy temperature.
For the policy improvement, we maximize the corresponding pessimistic surrogate:
\begin{align}
\label{eq:globalpi}
\pi_o^{t+1} ~\leftarrow~ \arg \max_{\pi}\;
\mathbb{E}_{\bfs \sim \mathcal{D}_o,~ \bfa \sim \pi(\cdot \mid \bfs)}
\Big[ \widehat{\Phi}^{\pi_o}_{K} Q(\bfs, \bfa) \Big],
\end{align}
where
$
\widehat{\Phi}^{\pi_o}_{K} Q(\bfs, \bfa)
= \min_{i=1,\ldots,2K} Q_{i} (\bfs, \bfa) - \beta \log \pi(\bfa \mid \bfs).
$
The server alternates \eqref{eq:globalQ} and \eqref{eq:globalpi} each round and broadcasts the updated $\pi_o^{t+1}$ and $Q_o^{t+1}$ to devices.

\begin{algorithm}[htbp]
    \caption{FORLER}
    \label{alg:forler}
    \LinesNumbered
    \KwIn{
    $\alpha_1$, $\alpha_2$ , $\mathcal{D}_o$, $\{\mathcal{D}_k\}^\mathcal{K}_{k=1}$}
    {Server initializes $\pi^0_o$ and $Q^0_o$ and sends agents\;}
        \For{$t=0$ \KwTo $T$}{
            \For{$k=1$ \KwTo $K$}{
                \tcp{Device side}
                Receive $Q^{t}_o$ and $\pi^{t}_o$ from server\\
                {Set $Q^0_k=Q^{t}_o$ and $\pi^0_k=\pi^{t}_o$}\\
                \For{$\tau=0$ \KwTo $\Gamma$}{
                    {Sample a random minibatch of $\bfs$ samples $(\bfs, \bfa, r, \bfs')$ from $\mathcal{B}_k$\\}
                    {Apply Algorithm \ref{alg:ar} with critic $Q^{\tau}_k$, $\bfs$, $\hat{\pi}_{\beta_k}^{\tau}$, and $\tau$ to obtain $\hat{\pi}_{\beta_k}^{\tau+1}$}\\
                    Calculate critic $Q^{\tau+1}_k$ according to Eq.~\eqref{eq:localQ}\\
                    Calculate policy $\pi^{\tau+1}_k$ according to Eq.~\eqref{eq:localpi}\\
                }
                device $k$ sends updated $Q^{\Gamma}_k$ to the server
            }
            \tcp{Server side}
            {Receive all local $\{Q^{\Gamma}_k\}^{\mathcal{K}}_{k=1}$}\\
            Calculate critic $Q_o^{t+1}$ according to Eq.~\eqref{eq:globalQ}\\
            Calculate policy $\pi_o^{t+1}$ according to Eq.~\eqref{eq:globalpi}\\
            Distribute $Q_o^{t+1}$ and $\pi_o^{t+1}$ to all devices\
        }
    \KwOut{Global policy ${\pi_o}$ and local policies $\{\pi_k\}^{\mathcal{K}}_{k=1}$}
\end{algorithm}

\subsection{FORLER}

Outlined in Algorithm \ref{alg:forler}, we design the FORLER algorithm for offline FRL, which follows a structured process integrating both server and device-side computations. 
On each device, at iteration $\tau$, a minibatch of transitions is sampled from the local dataset $\mathcal{D}_k$. The actor-rectification procedure (Algorithm~\ref{alg:ar}) then updates the local policy $\pi_k$ and critic $Q_k$ using the sampled data and the current rectified action $\hat{\pi}_{\beta_k}^{\tau}$. After the local step, each device returns its updated critic $Q_k^{\Gamma}$ to the server. The server aggregates the received critics $\{Q_k^{\Gamma}\}_{k=1}^{\mathcal{K}}$ to obtain the global critic $Q_o^{t+1}$ and policy $\pi_o^{t+1}$ according to the aggregation equations, and broadcasts them to all devices. This process repeats until convergence, yielding the final global policy $\pi_o$ and a set of optimized local policies $\{\pi_k\}_{k=1}^{\mathcal{K}}$.

\section{Theoretical Analysis}
In this section, we provide performance guarantees for FORLER and establish its safe policy improvement property.

To begin, we extend the lower-bound result of the CQL theorem to account for sampling error. Following prior work, we assume that the reward function and the transition dynamics satisfy standard concentration properties \cite{cql,qiao2025fova}. For any $(\bfs,\bfa)\in\mathcal{D}$, the following holds with probability at least $1-\delta$:

\begin{assumption}
\label{assump:concentration}
For all $(\bfs,\bfa)\in\mathcal{D}$, the reward $r$ and the transition kernel $\transitions$ satisfy, with probability $\ge 1-\delta$,
\begin{equation*}
\big|\, r - r(\bfs,\bfa)\,\big| \le \frac{C_{r,\delta}}{\sqrt{|\mathcal{D}(\bfs,\bfa)|}},
~
\big\| \hat{\transitions}(\bfs'|\bfs,\bfa) - \transitions(\bfs'|\bfs,\bfa) \big\|_{1}
\le
\frac{C_{\transitions,\delta}}{\sqrt{|\mathcal{D}(\bfs,\bfa)|}}.
\end{equation*}
\end{assumption}
$|\mathcal{D}(\bfs,\bfa)|$ denote the number of occurrences of $(\bfs,\bfa)$ in $\mathcal{D}$, let $\hat{\transitions}$ be the empirical transition kernel, and let $\|\cdot\|_{1}$ be the $L_{1}$ distance. Assume $|r(\bfs,\bfa)|\le R_{\max}$, which implies $\|\hat{Q}^k\|_{\infty}\le R_{\max}/(1-\gamma)$.
Using Assumption~\ref{assump:concentration}, we can bound the discrepancy between the empirical and true Bellman operators for any policy $\pi$:
\begin{align}
&~~\Big| \big(\hat{\bellman}^{\pi}\hat{Q}^{k}\big) - \big(\bellman^{\pi}\hat{Q}^{k}\big) \Big|\nonumber\\
&\le \big|\, r - r(\bfs,\bfa)\,\big|
+ \gamma \Big| \sum_{\bfs'} \big(\hat{\transitions}(\bfs'|\bfs,\bfa)-\transitions(\bfs'|\bfs,\bfa)\big)\,
\mathbb{E}_{\bfa'\sim\pi}\left[\hat{Q}^{k}(\bfs',\bfa')\right] \Big|
\nonumber\\
&\le \frac{C_{r,\delta} + \gamma\,C_{\transitions,\delta}\,\frac{2R_{\max}}{1-\gamma}}{\sqrt{|\mathcal{D}(\bfs,\bfa)|}}
= \frac{C_{r,\transitions,\delta}}{\sqrt{|\mathcal{D}(\bfs,\bfa)|}} .
\end{align}
This bound quantifies the potential overestimation induced by sampling error. The constant $C_{r,\transitions,\delta}$ is a function of $C_{r,\delta}$ and $C_{\transitions,\delta}$, and it inherits a $\sqrt{\log(1/\delta)}$ dependence. The analysis follows prior work that measures sampling effects via the empirical Bellman operator.
Based on these assumptions, we obtain the following safe policy improvement theorem.

\begin{theorem}[Safe policy improvement]
\label{thm:safe-pi}
Let $\pi_k^{*}$ be the policy obtained by optimizing Eq.~\eqref{eq:localpi}. Define
\[
D\left(\pi_k,\hat{\pi}_{\beta_k}\right)(\bfs) \triangleq
\frac{1-\hat{\pi}_{\beta_k}(\pi_k(\bfs)\mid \bfs)}{\hat{\pi}_{\beta_k}(\pi_k(\bfs)\mid \bfs)},
\]
and let $d^{\pi_k}(\bfs)$ denote the marginal discounted state-visitation distribution of policy $\pi_k$. Then, for $\alpha>0$ and $\eta\in(0,1)$,
\begin{align}
&J(\pi_k^{*}) - J(\hat{\pi}_{\beta_k})
\ge
\frac{\alpha}{1-\gamma}\,\mathbb{E}_{\bfs\sim d^{\pi_k^{*}}}\left[ D\left(\pi_k^{*},\hat{\pi}_{\beta_k}\right)(\bfs)\right]\nonumber\\
+&
\frac{\eta}{1-\eta}\,\mathbb{E}_{\bfs\sim d^{\pi_k^{*}}}\left[\big(\pi_k^{*}(\bfs)-\hat{\pi}_{\beta_k}(\bfs)\big)^{2}\right]\nonumber\\
-&
\frac{\eta}{1-\eta}\,\mathbb{E}_{\substack{\bfs\sim d^{\hat{\pi}_{\beta_k}}\\ \bfa\sim \hat{\pi}_{\beta_k}(\cdot\mid \bfs)}}\left[\big(a_k-\hat{\pi}_{\beta_k}(\bfs)\big)^{2}\right].
\end{align}
\end{theorem}
The inequality shows that under the stated conditions FEDORA achieves safe improvement over the behavior policy. The first term captures the benefit from deviating from the behavior policy, and the last two terms correspond to a quadratic penalty for the target policy and a baseline term under the behavior policy.
The complete proof is following:



\begin{figure*}[t]
\centering
\setlength{\abovecaptionskip}{0pt}
\setlength{\belowcaptionskip}{0pt}

\begin{minipage}[t]{0.64\textwidth}
  \vspace{-7pt}
  \centering
  \includegraphics[width=0.98\linewidth]{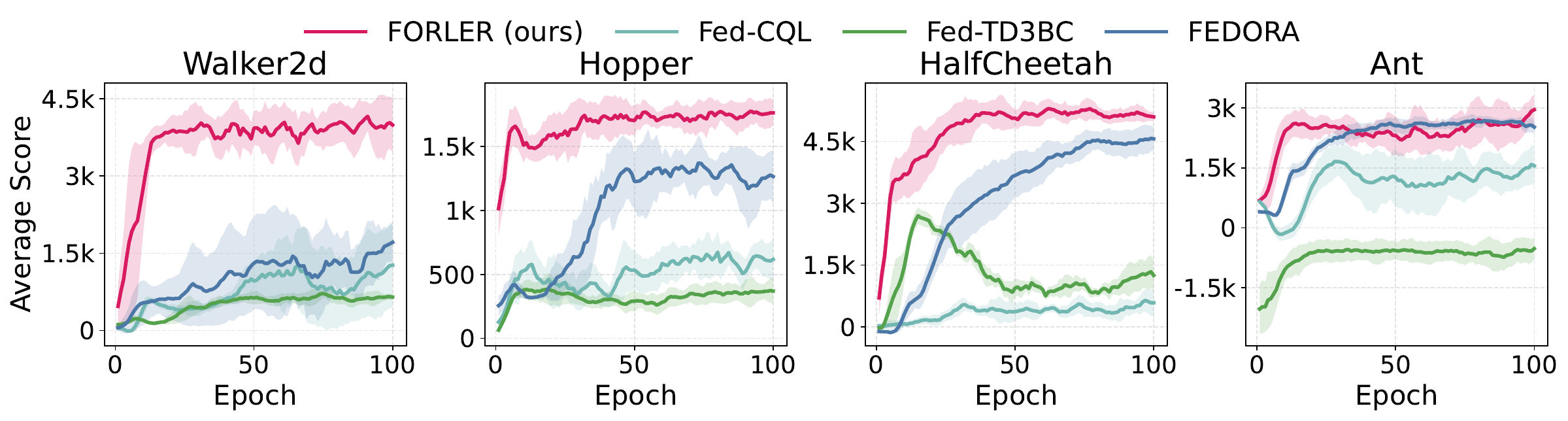}
  \caption{The performance of FORLER against baselines.}
  \label{fig:compare}
\end{minipage}
\begin{minipage}[t]{0.34\textwidth}
  \vspace{0pt}\centering
  \subfigure[\scriptsize Performance on devices]{
    \includegraphics[width=0.46\linewidth]{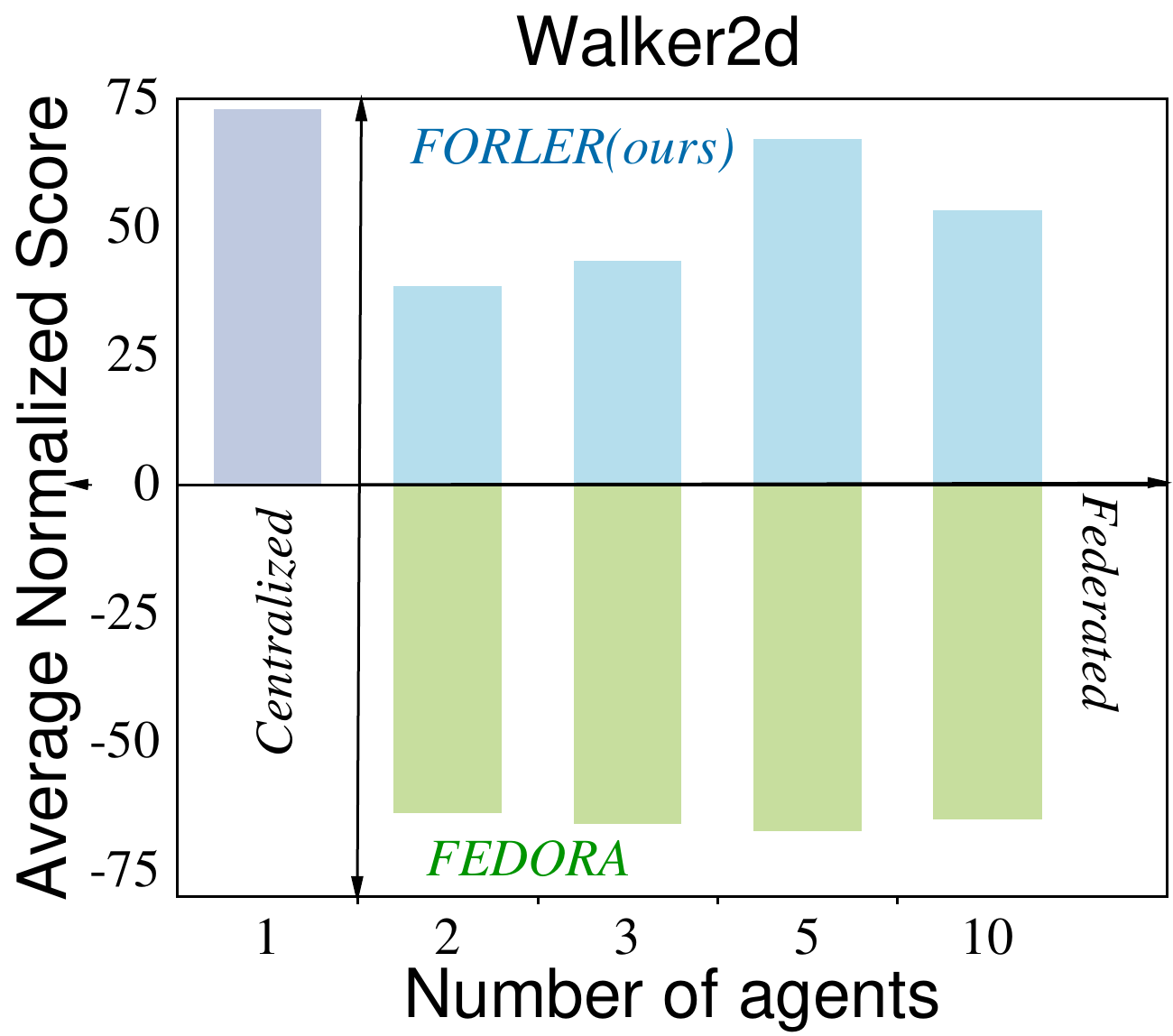}
    \label{fig:comparevsN}
  }\hfill
  \subfigure[\scriptsize Impact of devices count.]{
    \includegraphics[width=0.46\linewidth]{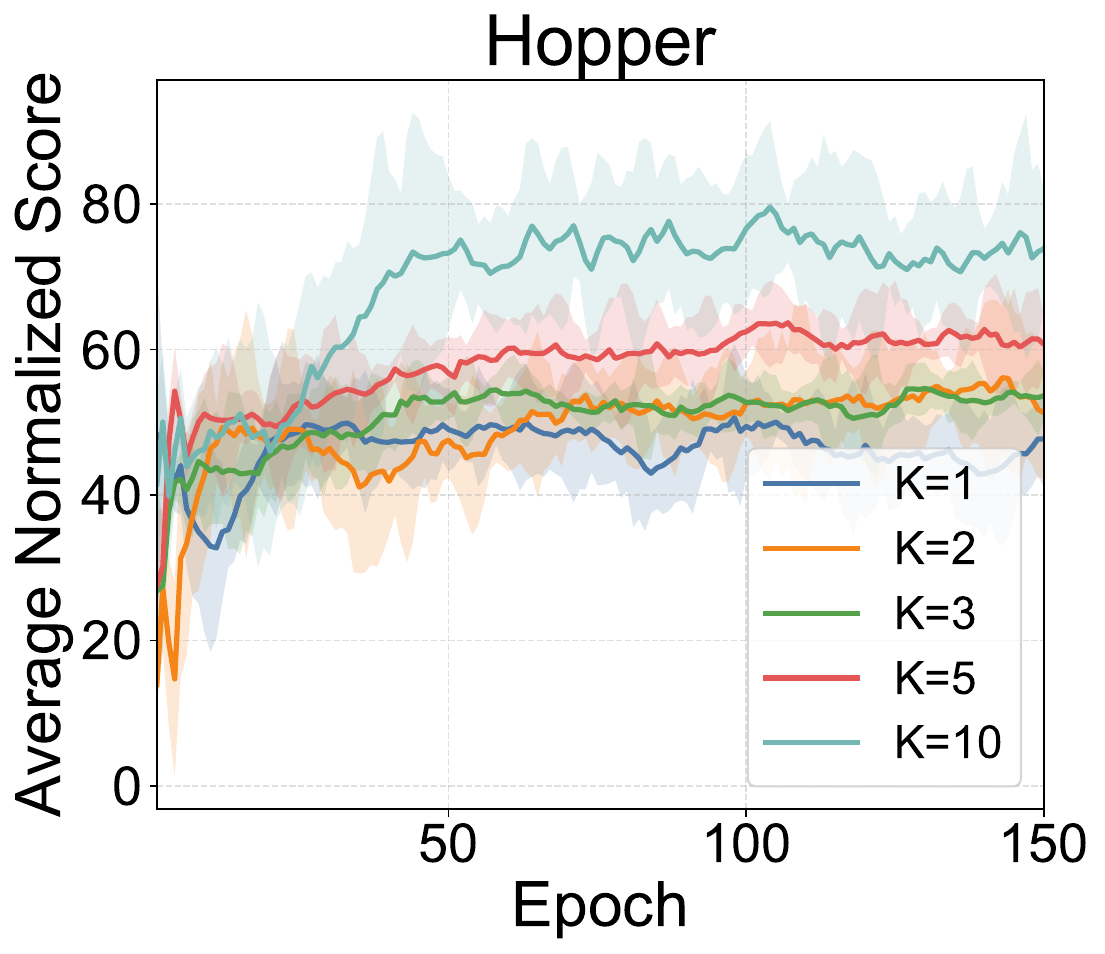}
    \label{fig:comparevsEpoch}
  }
  \vspace{-16pt}
\end{minipage}

\vspace{8pt}

\begin{minipage}[t]{0.65\textwidth}
  \vspace{0pt}\centering
  \includegraphics[width=0.98\linewidth]{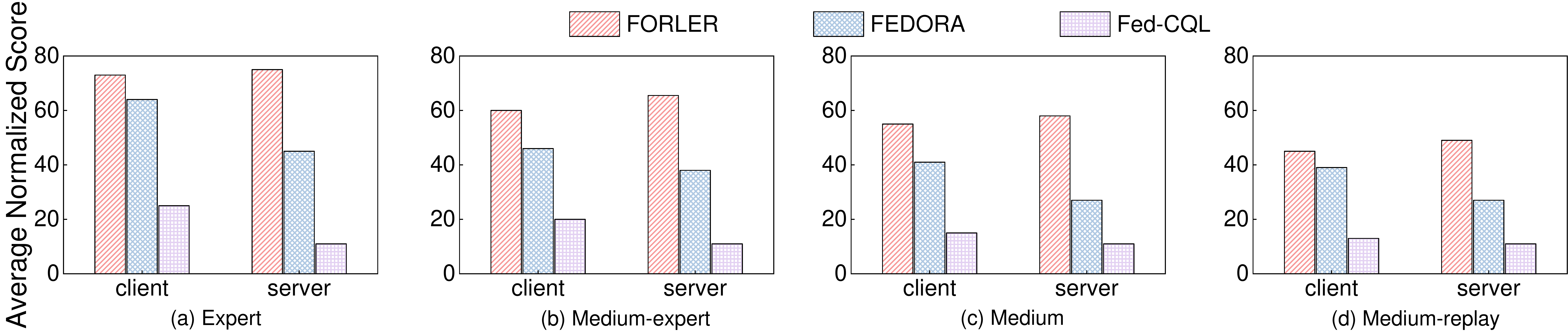}
  \caption{Performance comparison of different methods on policy pollution problem.}
  \label{fig:pl}
\end{minipage}
\begin{minipage}[t]{0.34\textwidth}
  \vspace{-0pt}\centering
  \subfigure[\scriptsize The impact of data size]{
    \includegraphics[width=0.46\linewidth]{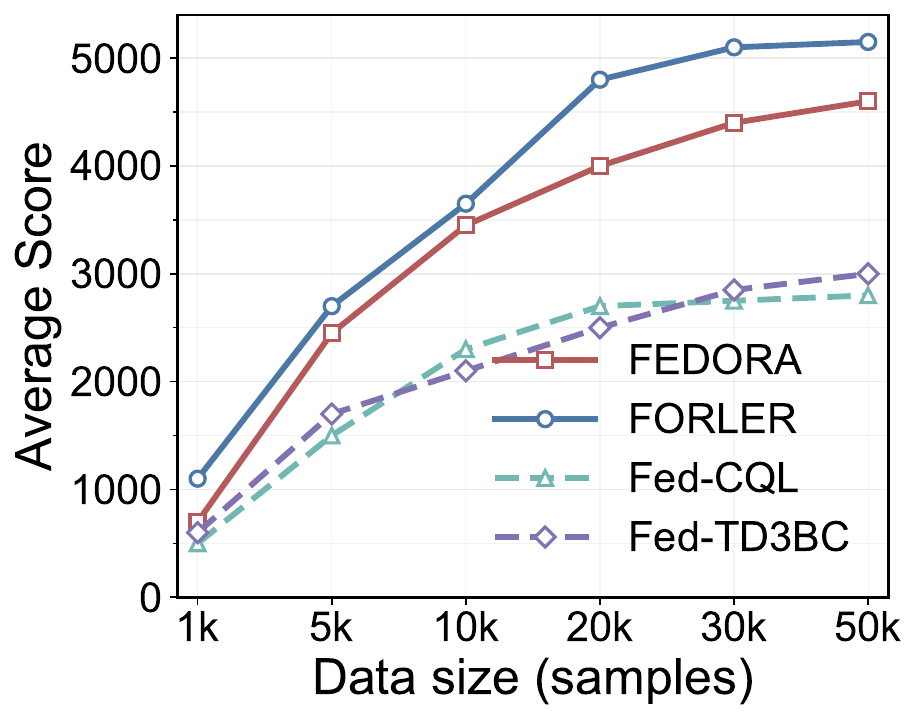}
    \label{fig:datasize}
  }\hfill
  \subfigure[\scriptsize The impact of parameters]{
    \includegraphics[width=0.46\linewidth]{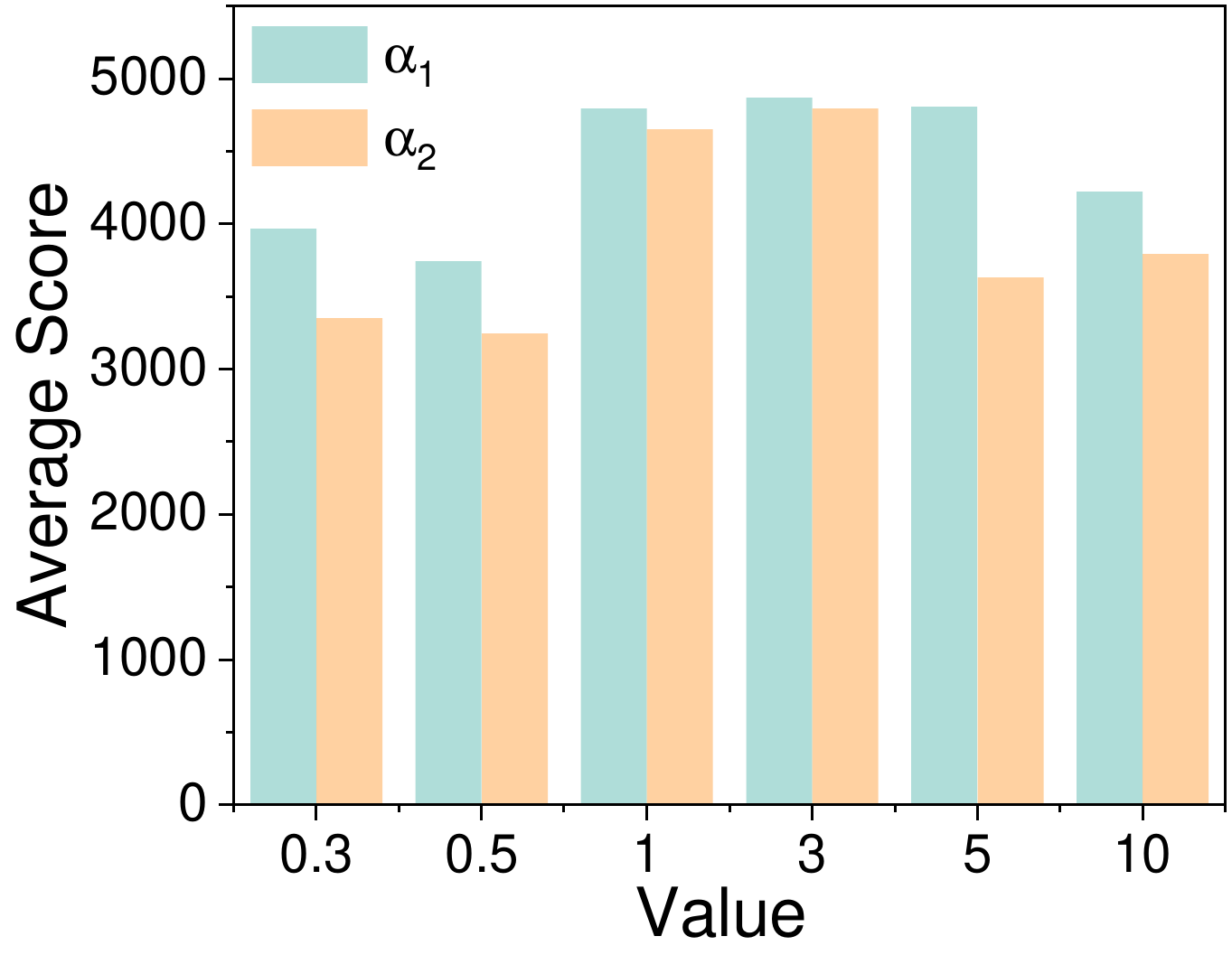}
    \label{fig:parameter}
  }

  \caption{Effect of devices, data, and hyperparameters.}
  \label{fig:hyperparameters}
\end{minipage}

\end{figure*}

\begin{proof}
For FORLER, we have the following iterative update:
\begin{align}
\label{cqlnew}
&{\arg\min}_{Q_k} 
\frac{1}{2} \mathbb{E}_{{\bfs}, {\bfa}, {\bfs}^{\prime} \sim \mathcal{D}}\left[\left(Q_k(\bfs, \bfa)-\hat{\mathcal{B}}^{\pi_k} \hat{Q}_{k}^{t}(\bfs, \bfa)\right)^{2}\right] \\
&+ \alpha \mathbb{E}_{{\bfs} \sim \mathcal{D}_k} \left[ \mathbb{E}_{{\bfa}\sim {\tilde{\pi}}_{k}} \left[Q_k({\bfs}, {\bfa})\right] - \mathbb{E}_{{\bfa} \sim \hat{\pi}_{\beta_k}} \left[Q_k({\bfs}, {\bfa})\right] \right]\nonumber
,
\end{align}
where $\tilde{\pi}_{k}({\bfa}|{\bfs})=1$ if and only if $a_k=\pi_k(\bfs).$
Let $\hat{Q}_k^{t+1}$ be the fixed point of solving Equation (\ref{cqlnew}) by setting the derivative of Eq. (\ref{cqlnew}) with respect to $Q_k$ to be $0$, then we have that
\begin{equation}
\hat{Q}_k^{t+1}(\bf\bfs,\bfa) = \hat{\mathcal{B}}^{\pi_k}\hat{Q}_{k}^{t}(\bf\bfs,\bfa)-\alpha \left(\frac{I_{a_k=\pi_k(\bfs)}}{\hat{\pi}_{\beta_k}(a_k|\bfs)}-1 \right),
\end{equation}
where $I$ is the indicator function.
Denote $D(\pi_k,\hat{\pi}_{\beta_k})(\bfs) = \frac{1}{\hat{\pi}_{\beta_k}(\pi_k(\bfs)|\bfs)}-1$, and we obtain the difference between the value function $\hat{V}_k(\bfs)$ and the original value function as:
$
\hat{V}_k(\bfs)={V}_k(\bfs)-\alpha D(\pi_k,\hat{\pi}_{\beta_k})(\bfs).
$
Then, the policy that minimizes the loss function defined in Eq. (\ref{eq:localpi}) is equivalently obtained by maximizing
\begin{equation}
\begin{split}
(1-\eta) \left(J(\pi_k)-\alpha \frac{1}{1-\gamma} \mathbb{E}_{\bfs\sim d^{\pi_k}_{\hat{M}_k}(\bfs)} \left[D(\pi_k,\hat{\pi}_{\beta_k})(\bfs)\right] \right) \\
- \eta \mathbb{E}_{\bfs\sim d^{\pi_k}_{\hat{M}_k}(\bfs)} \left[ (\pi_k(\bfs)-\hat{\pi}_{\beta_k})^2 \right].
\end{split}
\end{equation}
Therefore, we obtain that
\begin{equation}
\begin{split}
&(1-\eta) \left(J(\pi_k^*)-\alpha \frac{1}{1-\gamma} \mathbb{E}_{\bfs\sim d^{\pi_k^*}_{\hat{M}_k}(\bfs)} \left[D(\pi_k^*,\hat{\pi}_{\beta_k})(\bfs) \right] \right)\\
-&\eta \mathbb{E}_{\bfs\sim d^{\pi_k^*}_{\hat{M}_k}(\bfs)} \left[(\pi_k^*(\bfs)-\hat{\pi}_{\beta_k})^2 \right]\\
\geq & (1-\eta) J(\hat{\pi}_{\beta_k})
-\eta \mathbb{E}_{\bfs\sim d^{\hat{\pi}_{\beta_k}}_{\hat{M}_k}(\bfs),{\bfa} \sim \hat{\pi}_{\beta_k}({\bfa} \mid {\bfs})} \left[ (a_k-\hat{\pi}_{\beta_k})^2 \right].
\end{split}
\label{eq:thm1_final}
\end{equation}
Then, from Eq. (\ref{eq:thm1_final}) we obtain the result.
\end{proof}

\section{Simulation}\label{sec:simulation}
In this section, we evaluate our proposed algorithm by answering the following questions:
(1) How does the overall performance of our method compare to baselines?
(2) What is the effect of the number of devices, the size of the dataset, and critical hyperparameters?
(3) Is Algorithm FORLER effective in addressing the problem of policy pollution?
(4) What is the impact of the $\delta$-periodic strategy?

\subsection{Experimental Setup}\label{subsec:Exp-Setup}
We evaluate the benefits of our method (FORLER) using datasets from the widely recognized offline RL benchmark D4RL, which is built upon the MuJoCo simulation environment \cite{fu2020d4rl}. This benchmark encompasses 4 complex continuous control tasks (HalfCheetah, Walker2d, Ant, and Hopper).
Additionally, we compare FORLER with several baselines:
\begin{itemize}
\item 
\textit{Conservative Q-Learning (CQL) }, a wildly-studied offline RL algorithm, is executed in a centralized way \cite{kumar2020conservative}.
We re-run CQL using the code of OfflineRL-Kit 
\cite{offinerlkit}.
Centralized CQL is trained on the pooled union of all device datasets and serves as an upper bound reference.
\item 
\textit{Federated Conservative Q-Learning (Fed-CQL)}.
It works by changing the CQL rules locally on the devices with a set number of gradient steps and then combining the parameters in the FedAvg way at the server.
\item 
\textit{Federated TD3 with Behavior Cloning (Fed-TD3BC)}, a combination of FedAvg and TD3BC\cite{fujimoto2021minimalist}, operates by performing TD3BC updates locally on the devices with a certain number of gradient steps and subsequently aggregating the policy parameters at the server.
\item 
\textit{FEDORA} is the advanced offline FRL algorithm \cite{rengarajan2023federated}.
We reproduce FEDORA following their paper and the open-source implementations.
\end{itemize}

We implement FORLER in PyTorch 2.1.2 and run all experiments on Ubuntu 20.04.4 LTS with four NVIDIA GeForce RTX 3090 GPUs. 
Unless noted otherwise, all metrics are averaged over 3 independent runs with different random seeds.

\subsection{Experimental Evaluation}\label{subsec:Exp-Evaluation}

As shown in Fig.~\ref{fig:compare}, we compare FORLER with Fed-TD3BC, Fed-CQL, and FEDORA on D4RL MuJoCo tasks—Walker2d, Hopper, HalfCheetah, and Ant. FORLER achieves higher scores and faster convergence across all four tasks, consistently outperforming the baselines. The average-score curves indicate rapid, stable improvement and strong final policies, supporting FORLER as a robust and efficient solution for offline federated RL.


To assess the effects of device count, dataset size, and key hyperparameters on FORLER, we ran ablations summarized in Figs. \ref{fig:hyperparameters}.
In Fig. \ref{fig:comparevsN}, FORLER consistently outperforms FEDORA across agent counts, remaining robust as devices increase. In Fig. \ref{fig:comparevsEpoch}, FORLER maintains a steady advantage, while FEDORA varies more with the number of samples $N$. Furthermore, Fig. \ref{fig:datasize} shows FORLER scales well, achieving higher average scores and exploiting larger datasets more effectively. As shown in Fig. \ref{fig:parameter}, hyperparameter sweeps over $\alpha_1$ and $\alpha_2$ yield stable performance, indicating low sensitivity and minimal tuning needs. Overall, FORLER delivers superior, scalable performance with reduced dependence on device number, data volume, and hyperparameters.

To investigate whether FORLER mitigates policy pollution, we follow Fig.~\ref{fig:pl} and form a mixed-quality setup with six devices: four high-quality devices using expert, medium-expert, medium, and medium-replay data, and two low-quality devices using random data. We examine whether injecting random data degrades either high-quality devices or server-side aggregation. Across all four high-quality regimes, FORLER remains robust on both device and server sides, outperforming FEDORA and FedCQL and showing little to no degradation. High-quality devices retain strong scores despite random peers, and the aggregated policy is only minimally affected. Overall, FORLER effectively prevents policy pollution, preserving high-quality policy and fitting federated settings with uneven data quality.

\begin{figure}[ht!]
	\centering
	\subfigure[Impact of value $\delta$]{\label{fig:ARDelta}\includegraphics[width=0.21\textwidth]{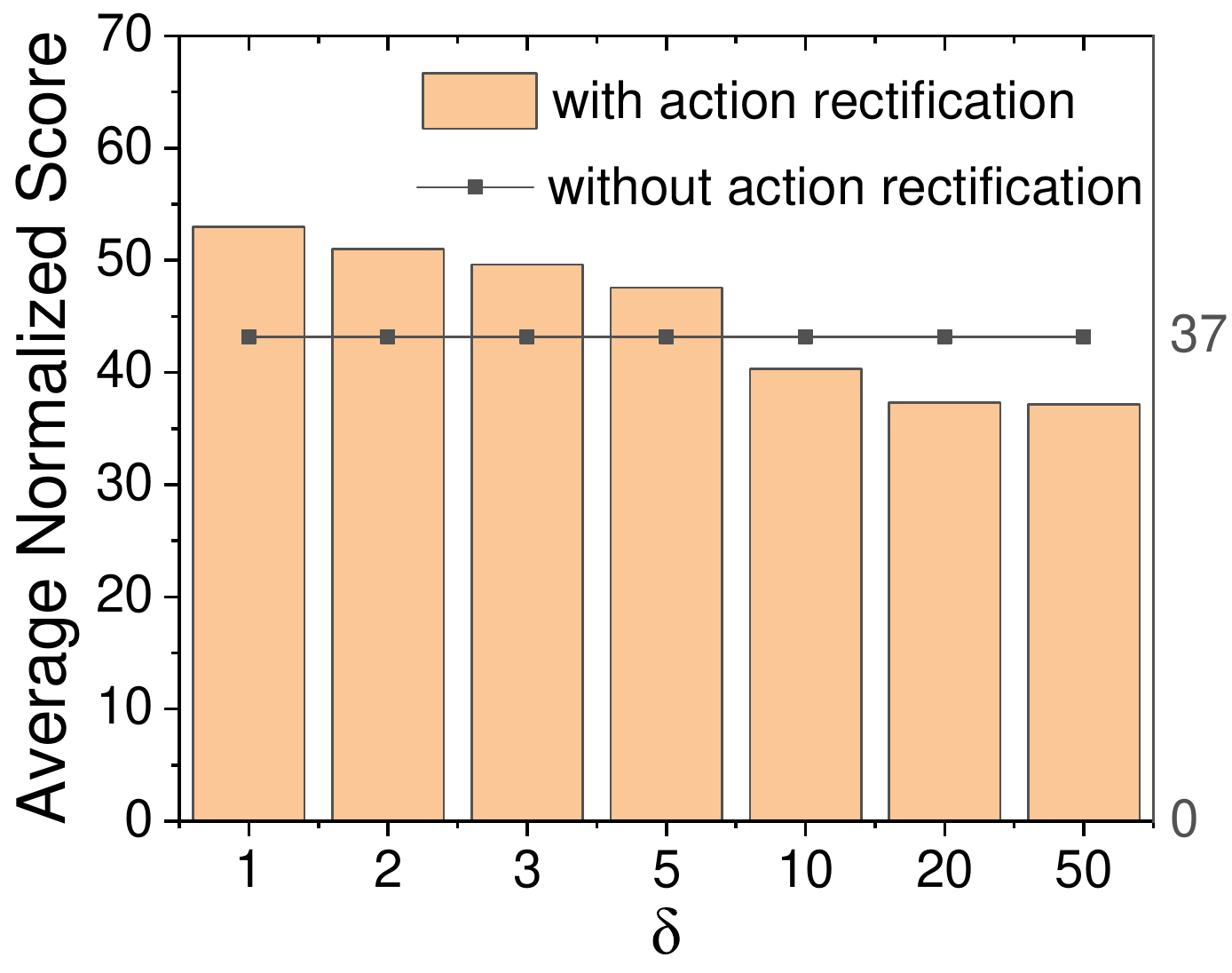}}
	\subfigure[Ablation study]{\label{fig:Ablation}\includegraphics[width=0.24\textwidth]{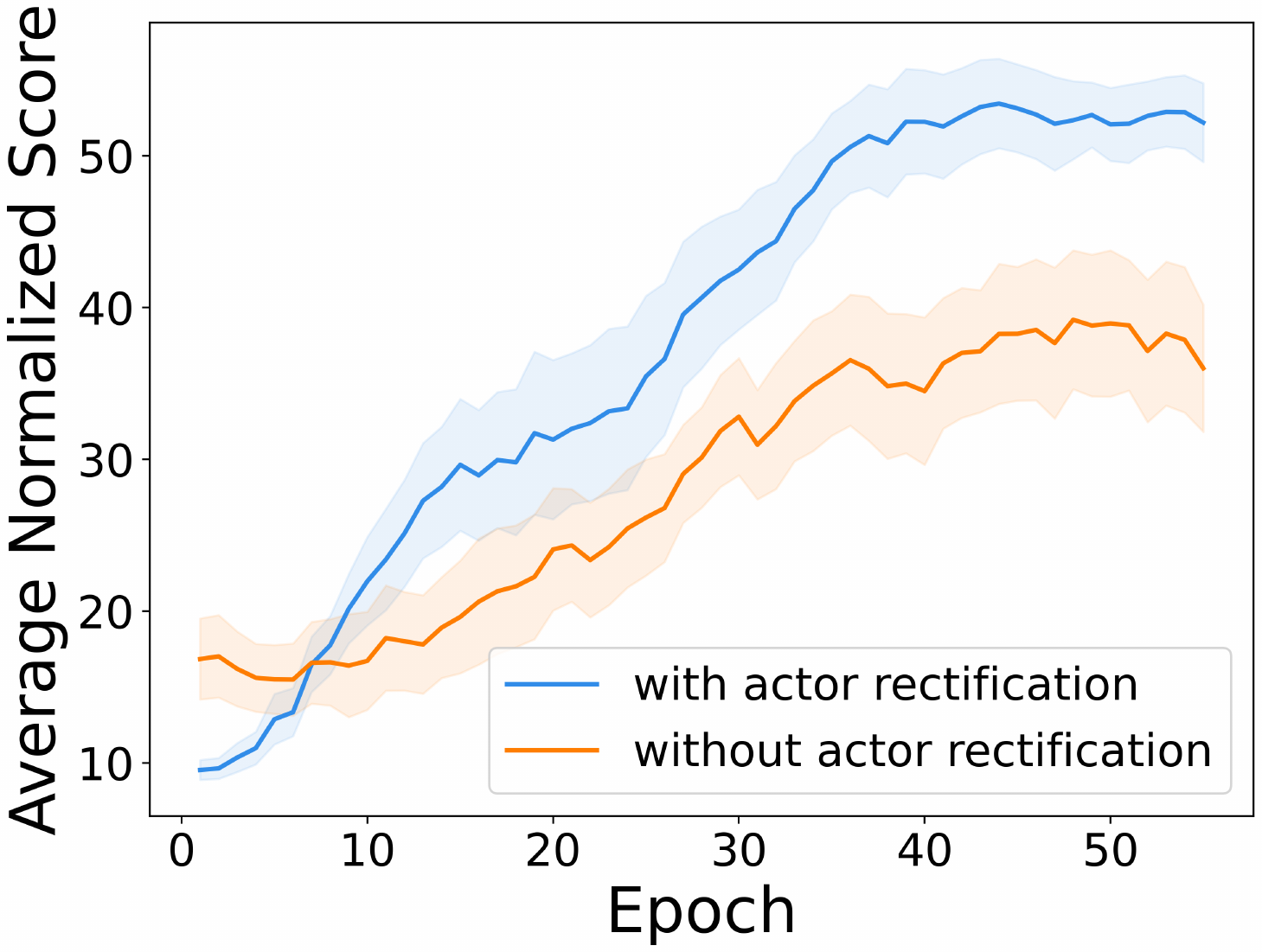}}
	\caption{Actor rectification improves performance: $\delta$ sensitivity and ablation.}
    \label{fig:ablation}
    \vspace{-1pt}
\end{figure}


In an ablation study depicted in Figure \ref{fig:ablation}, we examine the efficacy of actor rectification and the impact of the \(\delta\)-periodic strategy on the performance of our FORLER algorithm. 
As shown in Fig. \ref{fig:ARDelta}, we observe that implementing actor rectification consistently yields higher scores across all \(\delta\) values, underscoring its positive effect on the learning process. Moreover, the performance remains relatively stable across various \(\delta\) settings, suggesting that while the actor rectification contributes to performance gains, the specific choice of \(\delta\) within a certain range does not significantly alter the outcomes.
Fig. \ref{fig:Ablation} further reinforces the benefits of actor rectification, where the learning curve with rectification outpaces that without, indicating not only improved performance but also an accelerated rate of learning when actor rectification is applied.


\section{Conclusion}\label{sec:conclusion}
We introduce FORLER, an offline FRL method for IoT that couples server-side Q-ensemble aggregation with 
$\delta$-periodic actor rectification to improve robustness and local efficiency under data heterogeneity. Experiments on D4RL/MuJoCo confirm clear performance advantages over strong baselines. While our results are simulation-based, we include an explicit overhead analysis of the 
$\delta$-periodic strategy and leave real-world testbed validation and communication-aware optimizations for future work.

\section{ACKNOWLEDGMENT}
This research was supported in part by the National Natural Science Foundation of China under Grant 62572496, the Shenzhen Science and Technology Program under Grant JCYJ20250604175500001, and the Young Elite Scientist Sponsorship Program by CAST under Contract ZB2025-218.

\bibliographystyle{IEEEtran}
\bibliography{reference}

\begin{thebibliography}{10}
\providecommand{\url}[1]{#1}
\csname url@samestyle\endcsname
\providecommand{\newblock}{\relax}
\providecommand{\bibinfo}[2]{#2}
\providecommand{\BIBentrySTDinterwordspacing}{\spaceskip=0pt\relax}
\providecommand{\BIBentryALTinterwordstretchfactor}{4}
\providecommand{\BIBentryALTinterwordspacing}{\spaceskip=\fontdimen2\font plus
\BIBentryALTinterwordstretchfactor\fontdimen3\font minus
  \fontdimen4\font\relax}
\providecommand{\BIBforeignlanguage}[2]{{%
\expandafter\ifx\csname l@#1\endcsname\relax
\typeout{** WARNING: IEEEtran.bst: No hyphenation pattern has been}%
\typeout{** loaded for the language `#1'. Using the pattern for}%
\typeout{** the default language instead.}%
\else
\language=\csname l@#1\endcsname
\fi
#2}}
\providecommand{\BIBdecl}{\relax}
\BIBdecl

\bibitem{zhang2025carplanner}
D.~Zhang, J.~Liang, K.~Guo, S.~Lu, Q.~Wang, R.~Xiong, Z.~Miao, and Y.~Wang,
  ``Carplanner: Consistent auto-regressive trajectory planning for large-scale
  reinforcement learning in autonomous driving,'' in \emph{Proceedings of the
  Computer Vision and Pattern Recognition Conference}, 2025.

\bibitem{qiao2023popec}
N.~Qiao, S.~Yue, Y.~Zhang, and J.~Ren, ``Popec: Paoi-centric task offloading
  with priority over unreliable channels,'' \emph{IEEE/ACM Transactions on
  Networking}, vol.~32, no.~3, pp. 2376--2390, 2024.

\bibitem{Wang2025Squeezer}
X.~Wang, L.~Ma, Z.~Fu, X.~Li, Y.~Li, J.~Ren, Y.~Zhang, and Y.~Liu, ``Squeezer:
  Efficient multi-dnn inference for edge video analytics via cross-model
  scheduling,'' \emph{IEEE Transactions on Mobile Computing}, vol.~24, no.~12,
  pp. 13\,309--13\,321, 2025.

\bibitem{Zhang2023Distributed}
Y.~Zhang, C.~Ji, N.~Qiao, J.~Ren, Y.~Zhang, and Y.~Yang, ``Distributed pricing
  and bandwidth allocation in crowdsourced wireless community networks,''
  \emph{IEEE Transactions on Mobile Computing}, vol.~22, no.~9, pp. 5170--5183,
  2023.

\bibitem{wang2023uav}
L.~Wang, X.~Wu, Y.~Wang, Z.~Xiao, L.~Li, and A.~Fei, ``On uav serving node
  deployment for temporary coverage in forest environment: A hierarchical deep
  reinforcement learning approach,'' \emph{Chinese Journal of Electronics},
  vol.~32, no.~4, pp. 760--772, 2023.

\bibitem{seid2021multi}
A.~M. Seid, G.~O. Boateng, B.~Mareri, G.~Sun, and W.~Jiang, ``Multi-agent drl
  for task offloading and resource allocation in multi-uav enabled iot edge
  network,'' \emph{IEEE Transactions on Network and Service Management},
  vol.~18, no.~4, pp. 4531--4547, 2021.

\bibitem{qi2021federated}
J.~Qi, Q.~Zhou, L.~Lei, and K.~Zheng, ``Federated reinforcement learning:
  Techniques, applications, and open challenges,'' \emph{arXiv preprint
  arXiv:2108.11887}, 2021.

\bibitem{jin2022federated}
H.~Jin, Y.~Peng, W.~Yang, S.~Wang, and Z.~Zhang, ``Federated reinforcement
  learning with environment heterogeneity,'' in \emph{International Conference
  on Artificial Intelligence and Statistics}.\hskip 1em plus 0.5em minus
  0.4em\relax PMLR, 2022.

\bibitem{fang2025provably}
M.~Fang, X.~Wang, and N.~Z. Gong, ``Provably robust federated reinforcement
  learning,'' in \emph{Proceedings of the ACM on Web Conference 2025}, 2025,
  pp. 896--909.

\bibitem{levine2020offline}
S.~Levine, A.~Kumar, G.~Tucker, and J.~Fu, ``Offline reinforcement learning:
  Tutorial, review, and perspectives on open problems,'' \emph{arXiv preprint
  arXiv:2005.01643}, 2020.

\bibitem{qiao2025fova}
N.~Qiao, S.~Yue, J.~Ren, and Y.~Zhang, ``Fova: Offline federated reinforcement
  learning with mixed-quality data,'' \emph{IEEE Transactions on Networking},
  vol.~34, pp. 2031--2046, 2026.

\bibitem{rengarajan2023federated}
D.~Rengarajan, N.~Ragothaman, D.~Kalathil, and S.~Shakkottai, ``Federated
  ensemble-directed offline reinforcement learning,'' \emph{Advances in Neural
  Information Processing Systems}, vol.~37, pp. 6154--6179, 2024.

\bibitem{fujimoto2021minimalist}
S.~Fujimoto and S.~S. Gu, ``A minimalist approach to offline reinforcement
  learning,'' \emph{Advances in neural information processing systems},
  vol.~34, pp. 20\,132--20\,145, 2021.

\bibitem{an2021uncertainty}
G.~An, S.~Moon, J.-H. Kim, and H.~O. Song, ``Uncertainty-based offline
  reinforcement learning with diversified q-ensemble,'' \emph{Advances in
  neural information processing systems}, vol.~34, pp. 7436--7447, 2021.

\bibitem{qiao2026less}
\BIBentryALTinterwordspacing
Anonymous, ``Less is more: Clustered cross-covariance control for offline
  {RL},'' in \emph{The Fourteenth International Conference on Learning
  Representations}, 2026. [Online]. Available:
  \url{https://openreview.net/forum?id=drOy5wi6Qq}
\BIBentrySTDinterwordspacing

\bibitem{pan2022plan}
L.~Pan, L.~Huang, T.~Ma, and H.~Xu, ``Plan better amid conservatism: Offline
  multi-agent reinforcement learning with actor rectification,'' in
  \emph{International conference on machine learning}.\hskip 1em plus 0.5em
  minus 0.4em\relax PMLR, 2022.

\bibitem{kim2025penalizing}
J.~Kim, Y.~Shin, W.~Jung, S.~Hong, D.~Yoon, Y.~Sung, K.~Lee, and W.~Lim,
  ``Penalizing infeasible actions and reward scaling in reinforcement learning
  with offline data,'' \emph{arXiv preprint arXiv:2507.08761}, 2025.

\bibitem{cql}
A.~Kumar, A.~Zhou, G.~Tucker, and S.~Levine, ``Conservative q-learning for
  offline reinforcement learning,'' \emph{Advances in Neural Information
  Processing Systems}, vol.~33, pp. 1179--1191, 2020.

\bibitem{fu2020d4rl}
J.~Fu, A.~Kumar, O.~Nachum, G.~Tucker, and S.~Levine, ``D4rl: Datasets for deep
  data-driven reinforcement learning,'' \emph{arXiv preprint arXiv:2004.07219},
  2020.

\bibitem{kumar2020conservative}
A.~Kumar, A.~Zhou, G.~Tucker, and S.~Levine, ``Conservative q-learning for
  offline reinforcement learning,'' \emph{Advances in Neural Information
  Processing Systems}, vol.~33, pp. 1179--1191, 2020.

\bibitem{offinerlkit}
Y.~Sun, ``Offlinerl-kit: An elegant pytorch offline reinforcement learning
  library,'' \url{https://github.com/yihaosun1124/OfflineRL-Kit}, 2023.

\end{thebibliography}

\end{document}